\documentclass[11pt]{article}
\usepackage{microtype} 
\usepackage{booktabs}  
\usepackage{url}  
\usepackage{csquotes}
\usepackage{amsmath}
\usepackage{amsthm}
\usepackage[preprint]{automl}
\usepackage{natbib}
\usepackage{physics}
\usepackage{mathtools}
\newtheorem{proposition}{Proposition}

\newtheorem{assumption}{Assumption}

\newcommand{\argmin}{\mathop{\mathrm{argmin}}}
\newcommand{\xv}{\boldsymbol{x}}
\newcommand{\cv}{\boldsymbol{c}}
\newcommand{\D}{\mathcal{D}}
\newcommand{\Dl}{\D^{(l)}}
\newcommand{\Dg}{\D^{(g)}}

\newcommand{\Fg}{\mathcal{F}^{\gamma}}
\newcommand{\cistar}[1]{c_{#1}^{\star}}
\newcommand{\ceil}[1]{\lceil #1 \rceil}
\newcommand{\prob}{\mathbb{P}}
\newcommand{\rankeq}{\stackrel{\mathrm{rank}}{\simeq}}

\title{Optuna Constrained Tree-Structured Parzen Estimator \\ Is a Joint Density Generalization of c-TPE}

\author[1]{\nameemail{Shuhei Watanabe}{shuhei.watanabe.utokyo@gmail.com}}
\author[2]{\nameemail{Kaichi Irie}{irie.kaichi.24x@st.kyoto-u.ac.jp}}

\affil[1]{Independent Researcher}
\affil[2]{Kyoto University}

\hypersetup{%
  pdfauthor={Shuhei Watanabe, Kaichi Irie},
  pdftitle={Optuna Constrained Tree-Structured Parzen Estimator Is a Joint Density Generalization of c-TPE},
  pdfsubject={optuna-constrained-tree-structured-parzen-estimator-is-a-joint-density-generalization-of-c-tpe},
  pdfkeywords={AutoML, Bayesian Optimization, Black-Box Optimization, Hyperparameter Optimization}
}

\begin{document}

\maketitle

\begin{abstract}
  Constrained hyperparameter optimization (HPO) is common in practice, yet Optuna's widely used constrained TPE lacks algorithmic analysis.
  While c-TPE proposes an expected constrained improvement (ECI) approach assuming \textit{independence} between the objective and constraints, Optuna uses a single joint density over both.
  We show that Optuna's constrained TPE is \textit{joint c-TPE}---the same ECI acquisition function using a joint likelihood.
  We demonstrate joint c-TPE is invariant to constraint duplication whereas independent c-TPE degrades as the product accumulates duplicated factors.
  We outline practical tradeoffs between the formulations and directions for future study.
\end{abstract}

\section{Introduction}
The performance of deep learning algorithms is sensitive to hyperparameter (HP) selection, which can be formulated as an optimization problem $\xv^\star \in \argmin_{\xv} f(\xv)$.
Optuna~\citep{akiba2019optuna} is a de facto HPO framework whose default algorithm is tree-structured Parzen estimator (TPE)~\citep{bergstra2011algorithms,watanabe2023tree}.
TPE models the likelihood of HPs given observations $\D = \{(\xv_n, f_n)\}_{n=1}^{N}$, where $\xv_n$ is the $n$-th HP, e.g., learning rate, and $f_n$ is the corresponding objective value, e.g., error rate, via a split at the top-$\gamma$ quantile:
\vspace{-2.5mm}
\begin{equation}
  p(\xv | f, \D) = \begin{cases}
    p(\xv | \Dl) & (f \le f^\gamma) \\
    p(\xv | \Dg) & (f > f^\gamma)
  \end{cases}
  \label{eq:tpe-split}
  \vspace{-2.5mm}
\end{equation}
where $f^\gamma$ is the top-$\gamma$ quantile objective value, $\Dl$ and $\Dg$ contain observations below (better) and above (worse) $f^\gamma$, and both densities are modeled by Parzen estimators, also known as kernel density estimators (KDEs).
\citet{watanabe2022c} showed that the density ratio is proportional to the probability of improvement $p(\xv|\Dl)/p(\xv|\Dg) \propto \prob(f \leq f^\gamma \mid \xv)$.

While the aforementioned TPE formulation solves single-objective HPO nicely, real-world HPO often involves black-box constraints such as memory budget or inference latency.
Previously, \citet{watanabe2022c} proposed a constrained TPE using expected constrained improvement (ECI)~\citep{gardner2014bayesian,gelbart2014bayesian}, assuming \textit{independence} between the objective and constraints:
\vspace{-1mm}
\begin{equation}
  \mathrm{ECI}_{f^\gamma}[\xv | \cv^\star, \D]
  = \mathrm{EI}_{f^\gamma}[\xv | \D] \prod_{i=1}^{C} \prob(c_i \le \cistar{i} | \xv, \D)
  \label{eq:eci}
\vspace{-1mm}
\end{equation}
where $c_i(\xv) \le \cistar{i}$ for $i \in \{1, \ldots, C\}$ are black-box constraints.

Optuna's TPE also supports constrained optimization, but its implementation has no formal documentation and differs from the method known as c-TPE in the literature.
Optuna builds only one pair of KDEs, whereas \citet{watanabe2022c} build $C+1$ pairs of KDEs.
Despite this implementation gap, we found that both approaches share the same theoretical foundation of expected constrained improvement (ECI).
Accordingly, we narrow the definition of c-TPE given in~\citet{watanabe2022c} to formulations that specifically use ECI as their acquisition function.
Under this definition, we show that Optuna's constrained TPE, dubbed \textit{joint c-TPE}, is a joint density generalization of the original c-TPE, dubbed \textit{independent c-TPE}.
Note that although independent c-TPE is also available via OptunaHub~\citep{ozaki2025optunahub}, this algorithm is not included in the main package.
We empirically demonstrate the key advantage of joint c-TPE, which avoids the relative importance dilution that independent formulations suffer under constraint correlation.
Finally, we discuss the tradeoffs and unexplored open problems in joint c-TPE.

\section{Findings: Optuna Constrained TPE Is a Joint Density Generalization of c-TPE}\label{sec:finding}

In this paper, we analyze the implementation present in Optuna v4.8.
To formalize this approach, we rely on the following assumption regarding the marginal constraint density $p(\cv | \D) = \int p(f, \cv | \D) df$ in the feasible region:
\begin{assumption}
  Given the feasible region $\Omega \coloneqq \prod_{i=1}^{C} (-\infty, \cistar{i}]$, $\gamma \leq \int_{\cv \in \Omega} p(\cv | \D)\, d\cv$ holds.
  \label{assumption:gamma-feasibility}
\end{assumption}
Informally, this assumption requires that the feasible region contains at least a fraction $\gamma$ of observations where $\cv \in \Omega$, ensuring sufficient feasible observations exist in $\Dl$.
Section~\ref{sec:discussion} discusses limitations of this assumption.
The Optuna approach first splits observations $\D = \{(\xv_n, f_n, \cv_n)\}_{n=1}^N$ into feasible $\D_{\mathrm{feas}} = \{(\xv_n, f_n, \cv_n) \in \D \mid \cv \in \Omega\}$, and infeasible set $\D \setminus \D_{\mathrm{feas}}$.
Then it sorts the feasible observations by objective value $f_n$ and fills $\Dl$ with the top $\ceil{\gamma N}$ feasible observations.
The next HP is selected by maximizing the density ratio $p(\xv|\Dl)/p(\xv|\Dg)$ where $\Dg \coloneqq \D \setminus \Dl$.
Importantly, the Optuna approach computes the likelihood of the objective and constraints \textit{jointly}:
\begin{equation}
  p(\xv | f, \cv, \D) = \begin{cases}
    p(\xv | \Dl) & (f \le f^\gamma \text{ and } \cv \in \Omega) \\
    p(\xv | \Dg) & (\text{otherwise})
  \end{cases}.
  \label{eq:joint-split}
\end{equation}
This formulation induces the following proposition:
\begin{proposition}\label{prop:eci}
Under the joint likelihood in Eq.~\eqref{eq:joint-split}, $\mathrm{ECI}_{f^\gamma}[\xv | \cv^\star, \D] \rankeq p(\xv | \Dl) / p(\xv | \Dg)$, where rank equivalence $f(\xv) \rankeq g(\xv)$ means the ordering is preserved: $f(\xv) \leq f(\xv^\prime) \Leftrightarrow g(\xv) \leq g(\xv^\prime)$.
\end{proposition}

\begin{proof}
Applying Bayes' rule to $\mathrm{ECI}_{f^\gamma}[\xv \mid \cv^\star, \D] \propto \prob(f \leq f^\gamma, \cv \in \Omega \mid \xv, \D)$ shown by \citet{watanabe2022c}:
\begin{align}
  \mathrm{ECI}_{f^\gamma}[\xv \mid \cv^\star, \D] \propto \prob(f \le f^\gamma, \cv \in \Omega \mid \xv, \D) \nonumber &= \int_{(f,\cv) \in \Fg \times \Omega}
    \frac{p(\xv | f, \cv, \D)\, p(f, \cv | \D)}{p(\xv | \D)}\, df\, d\cv \nonumber \\
  &= \frac{p(\xv | \Dl)}{p(\xv | \D)}
    \underbrace{\int_{(f,\cv) \in \Fg \times \Omega}
    p(f, \cv | \D)\, df\, d\cv}_{= \gamma} \nonumber \\
  &\propto \qty(\gamma + (1-\gamma)
  \frac{p(\xv|\Dg)}{p(\xv|\Dl)})^{-1}
  \rankeq \frac{p(\xv|\Dl)}{p(\xv|\Dg)}
  \label{eq:joint-af}
\end{align}
where $\Fg \coloneqq (-\infty, f^\gamma]$ is the $\gamma$-quantile objective region on the feasible region, and the last line uses $p(\xv|\D) = \gamma\, p(\xv|\Dl) + (1-\gamma)\, p(\xv|\Dg)$, obtained by marginalizing the joint likelihood in Eq.~\eqref{eq:joint-split} over all $(f, \cv)$.
\end{proof}

This result has three notable properties.
First, when all observations are feasible, the joint split shown in Eq.~\eqref{eq:joint-split} reduces to the standard TPE split shown in Eq.~\eqref{eq:tpe-split}.
Second, the $\gamma$-quantile value $f^\gamma$ constrained on the feasible region is identical to c-TPE's split for the objective component.
Third, and most importantly, the joint formulation uses one KDE pair whose split accounts for all constraints simultaneously, whereas independent c-TPE uses $C+1$ pairs and treats each constraint marginally.
This means duplicating a constraint does not change the feasible/infeasible partition under joint c-TPE, making the acquisition function invariant to constraint duplication unlike independent c-TPE.
We empirically verify the last property in the next section.

\section{Empirical Evaluations}\label{sec:experiments}

We compare independent and joint c-TPE on a 2D toy problem, minimizing $f(x, y) = (x-2)^2 + (y-2)^2$ subject to two constraints $c_1(x, y) = x \leq 0, c_2(x, y) = y \leq 0$ over $[-5, 5]^2$.
We consider two scenarios: (1)~independent constraints $c_1, c_2$, and (2)~$c_1$ and many duplicated copies of $c_2$ (perfectly correlated constraints).
Comparing the scenario with many duplicated constraints (\textbf{Bottom}) to the independent constraint scenario (Figure~\ref{fig:results} (\textbf{Top})), independent c-TPE degrades because each duplicate adds a factor to the product of relative density ratios, accumulating constraint information and overweighting feasibility along the $y$-axis relative to the $x$-axis.
Meanwhile, joint c-TPE is invariant to duplication since the feasible/infeasible partition remains unchanged.
Note, however, that independent c-TPE outperformed joint c-TPE in the independent constraint scenario because the scenario aligns better with the independent formulation.

\begin{figure}[t]
  \centering
  \includegraphics[width=0.92\linewidth]{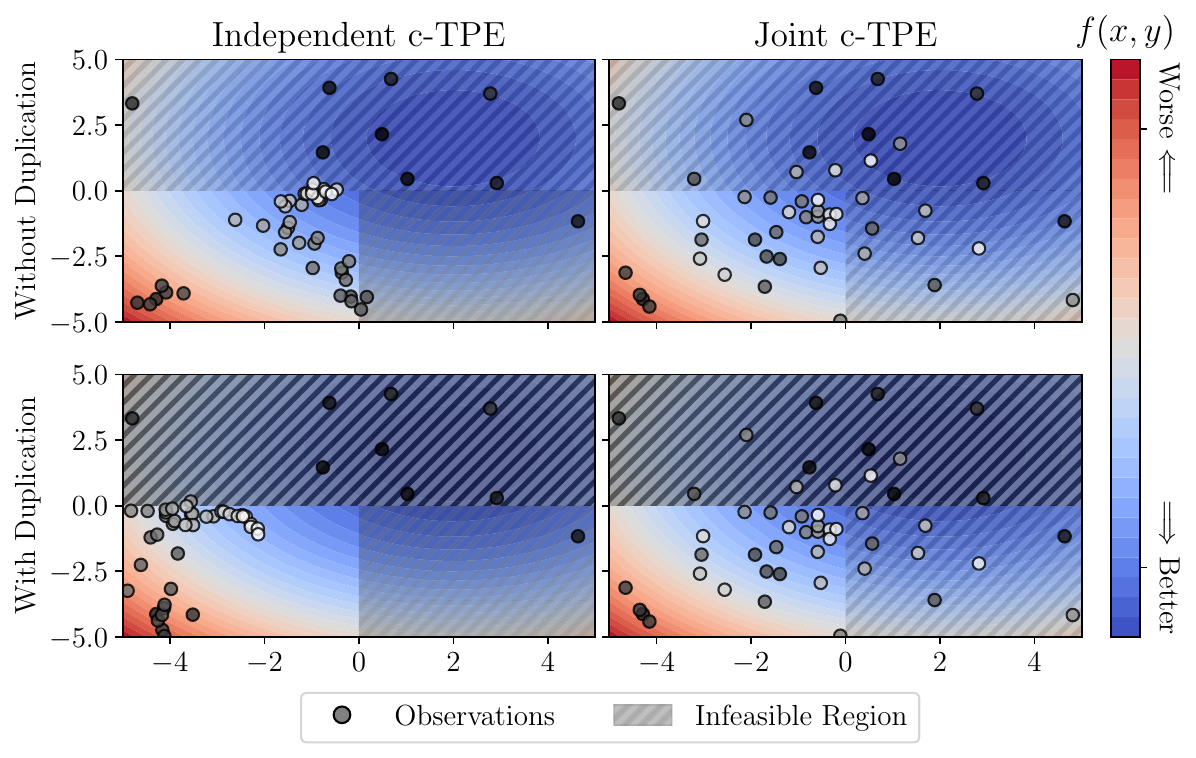}
  \vspace{-2mm}
  \caption{
    The comparison of independent c-TPE (\textbf{Left}) and joint c-TPE (\textbf{Right}) on a 2D problem with $50$ observations.
    The shaded area shows the feasible region, and the color gradation shows the objective value, which is better when the color is darker.
    Each dot represents an observation.
    The dots are colored based on the observation order; black means early observations, and white means later observations.
    \textbf{Top}: The case without duplicated constraints.
    Independent c-TPE explores at the boundary aggressively by exploiting the independence of the constraint pair.
    \textbf{Bottom}: The case with many duplicated constraints; its infeasible region shows much darker coloring.
    While joint c-TPE exhibits the identical sampling behavior, independent c-TPE shows highly conservative exploration along the $x$-axis because the feasibility along the $y$-axis is overweighted.
  }
  \label{fig:results}
\end{figure}

\section{Limitations \& Conclusion}\label{sec:discussion}

A primary methodological limitation of joint c-TPE comes from Assumption~\ref{assumption:gamma-feasibility}.
When Assumption~\ref{assumption:gamma-feasibility} is violated, Optuna fills empty slots in $\Dl$ from infeasible observations sorted by total constraint violation~\citep{deb2002fast}.
However, the effectiveness and theoretical validity of this fallback strategy remain largely unexplored, unlike independent c-TPE which has been carefully engineered and analyzed by \citet{watanabe2022c}.
Open questions include: (1) which tie-breaking strategies are viable, (2) which performs best, and (3) when joint c-TPE outperforms independent c-TPE (and when not).
A key structural difference is that independent c-TPE requires only one feasible solution per constraint individually, whereas joint c-TPE requires feasibility across \emph{all} constraints simultaneously.
This stricter feasibility requirement can be problematic in independently or severely constrained settings.
Additionally, joint c-TPE cannot structurally accommodate partial observations (see Appendix C of \citet{watanabe2022c}), which may slow convergence when some constraints are cheaper to evaluate than the objective.

On the other hand, joint c-TPE offers complementary advantages.
It is structurally simpler and naturally invariant to constraint duplication---a property independent c-TPE lacks.
Since joint c-TPE maintains only one KDE pair rather than $C+1$ pairs, it aligns naturally with multi-objective TPE formulations~\citep{ozaki2020multiobjective,ozaki2022multiobjective}, facilitating future extensions.
Determining when each formulation is preferable remains an open question, requiring systematic empirical comparison across diverse problem structures as future work.

\paragraph{Citation Guidance}
When using Optuna's constrained TPE, we recommend citing both this paper and \citet{watanabe2022c}, as the latter provides a more detailed treatment of the constrained TPE mechanism and problem setup.

\bibliography{ref}

\begin{thebibliography}{}

\bibitem[Akiba et~al., 2019]{akiba2019optuna}
Akiba, T., Sano, S., Yanase, T., Ohta, T., and Koyama, M. (2019).
\newblock {Optuna}: A next-generation hyperparameter optimization framework.
\newblock In {\em ACM SIGKDD International Conference on Knowledge Discovery \&
  Data Mining}.

\bibitem[Bergstra et~al., 2011]{bergstra2011algorithms}
Bergstra, J., Bardenet, R., Bengio, Y., and K{\'e}gl, B. (2011).
\newblock Algorithms for hyper-parameter optimization.
\newblock In {\em Advances in Neural Information Processing Systems}.

\bibitem[Deb et~al., 2002]{deb2002fast}
Deb, K., Pratap, A., Agarwal, S., and Meyarivan, T. (2002).
\newblock A fast and elitist multiobjective genetic algorithm: {NSGA-II}.
\newblock {\em IEEE Transactions on Evolutionary Computation}, 6(2).

\bibitem[Gardner et~al., 2014]{gardner2014bayesian}
Gardner, J.~R., Kusner, M.~J., Xu, Z.~E., Weinberger, K.~Q., and Cunningham,
  J.~P. (2014).
\newblock {B}ayesian optimization with inequality constraints.
\newblock In {\em International Conference on Machine Learning}.

\bibitem[Gelbart et~al., 2014]{gelbart2014bayesian}
Gelbart, M.~A., Snoek, J., and Adams, R.~P. (2014).
\newblock {B}ayesian optimization with unknown constraints.
\newblock In {\em Uncertainty in Artificial Intelligence}.

\bibitem[Ozaki et~al., 2022]{ozaki2022multiobjective}
Ozaki, Y., Tanigaki, Y., Watanabe, S., Nomura, M., and Onishi, M. (2022).
\newblock Multiobjective tree-structured {P}arzen estimator.
\newblock {\em Journal of Artificial Intelligence Research}, 73.

\bibitem[Ozaki et~al., 2020]{ozaki2020multiobjective}
Ozaki, Y., Tanigaki, Y., Watanabe, S., and Onishi, M. (2020).
\newblock Multiobjective tree-structured {P}arzen estimator for computationally
  expensive optimization problems.
\newblock In {\em Genetic and Evolutionary Computation Conference}.

\bibitem[Ozaki et~al., 2026]{ozaki2025optunahub}
Ozaki, Y., Watanabe, S., and Yanase, T. (2026).
\newblock {OptunaHub}: A platform for black-box optimization.
\newblock {\em Journal of Machine Learning Research}.

\bibitem[Watanabe, 2023]{watanabe2023tree}
Watanabe, S. (2023).
\newblock Tree-structured {P}arzen estimator: Understanding its algorithm
  components and their roles for better empirical performance.
\newblock {\em arXiv preprint arXiv:2304.11127}.

\bibitem[Watanabe and Hutter, 2023]{watanabe2022c}
Watanabe, S. and Hutter, F. (2023).
\newblock {c-TPE}: Tree-structured {P}arzen estimator with inequality
  constraints for expensive hyperparameter optimization.
\newblock In {\em International Joint Conference on Artificial Intelligence}.

\end{thebibliography}

\end{document}